%% file: imageSearchByShapes_review.tex
\documentclass[10pt,letterpaper]{article}
\pdfoutput=1
\usepackage{times}
\usepackage{epsfig}
\usepackage{graphicx}
\usepackage{caption}
\usepackage{subcaption}
\usepackage{amsmath}
\usepackage{amssymb}
\usepackage{amsthm}
\usepackage{color}
\usepackage{longtable}

\usepackage[pagebackref=true,breaklinks=true,letterpaper=true,colorlinks,bookmarks=false]{hyperref}

\newtheorem{thm}{Theorem}

\newtheorem{lem}[thm]{Lemma}

\newtheorem{defn}{Definition}

\newtheorem{hypo}{Assumption}
\let\vec=\mathbf
\let\mat=\mathbf
\let\set=\mathcal

\def \saliency {\textup{\saliency}}

\def \path {\mathit{path}}

\def \label {\mathit{label}}

\newcommand{\comment}[1]{}

\newcommand{\degree}{\ensuremath{^\circ}}

\newcommand{\mypara}{\paragraph}

\newcommand{\denselist}{\itemsep 3pt\parsep=3pt\partopsep 3pt\vspace{-\topsep}}
\newcommand{\bitem}{\begin{itemize}\denselist}
	\newcommand{\eitem}{\end{itemize}}
\newcommand{\benum}{\begin{enumerate}\denselist}
	\newcommand{\eenum}{\end{enumerate}}
\newcommand{\bdescr}{\begin{description}\denselist}
	\newcommand{\edescr}{\end{description}}

\begin{document}

%%%%%%%%% TITLE
% \title{Across-view Imaged Object Comparison by 3D Shape Network}
\title{3D-Assisted Image Feature Synthesis for Novel Views of an Object}

\author{Hao Su\footnote{indicates equal contributions.}\qquad Fan Wang$^{*}$ \qquad Li Yi \qquad Leonidas Guibas\\
{Stanford University}
}

\maketitle
%\thispagestyle{empty}

%%%%%%%%% ABSTRACT
\begin{abstract}
Comparing two images in a view-invariant way has been a challenging problem in computer vision for a long time, as visual features are not stable under large view point changes. In this paper, given a single input image of an object, we synthesize new features for other views of the same object. To accomplish this, we introduce an aligned set of 3D models in the same class as the input object image. Each 3D model is represented by a set of views, and we study the correlation of image patches between different views, seeking what we call surrogates --- patches in one view whose feature content predicts well the features of a patch in another view. In particular, for each patch in the novel desired view, we seek surrogates from the observed view of the given image. For a given surrogate, we predict that surrogate using linear combination of the corresponding patches of the 3D model views, learn the coefficients, and then transfer these coefficients on a per patch basis to synthesize the features of the patch in the novel view. In this way we can create feature sets for all views of the latent object, providing us a multi-view representation of the object. View-invariant object comparisons are achieved simply by computing the $L^2$ distances between the features of corresponding views. We provide theoretical and empirical analysis of the feature synthesis process, and evaluate the proposed view-agnostic distance (VAD) in fine-grained image retrieval (100 object classes) and classification tasks. Experimental results show that our synthesized features do enable view-independent comparison between images and perform significantly better than traditional image features in this respect.

\end{abstract}

%%%%%%%%% BODY TEXT
\section{Introduction and Related Work}
\input{intro}

% \section{Related Work}
% \input{related}

%\section{multiview Shape Representation by Multi-view Image Features}
% \input{representation}

\section{Problem Formulation and Method Overview}
\input{formulation}

\input{overview}

\input{approach}

\section{Experiments}
\input{experiment}

\section{Conclusion and Future Work}
\input{conclusion}
\input{limit}

\bibliographystyle{ieee}
\bibliography{ComputerVision,imageShapeNet,NIPS2012}

\section{Appendix}
\input{supplementary}

\end{document}

%% file: intro.tex
Object recognition plays a key role in many applications and has been a central topic in the vision community. The fundamental question --- easy to state, albeit harder to formalize --- is: when do two or more images are said to be ``same'' or ``different''? Appearance differences of images can be factored as intrinsic or extrinsic. Intrinsic factors refer to properties of the imaged objects themselves, such as differences in the topology, geometry and material of the underlying or latent imaged shapes. Extrinsic factors on the other hand include illumination, viewpoint, and more generally visibility effects (occlusion, etc.). For object recognition and image understanding, it is a fundamental problem to find the ``best'' representation which only focuses on intrinsic object properties and is invariant to extrinsic factors (Fig.~\ref{fig:teaser}).
% This has lots of real-world applications, for example, in online shopping, users pay more attention to the difference of product themselves, rather than how they are pictured.

\begin{figure}
	\centering
\includegraphics[width=0.8\linewidth]{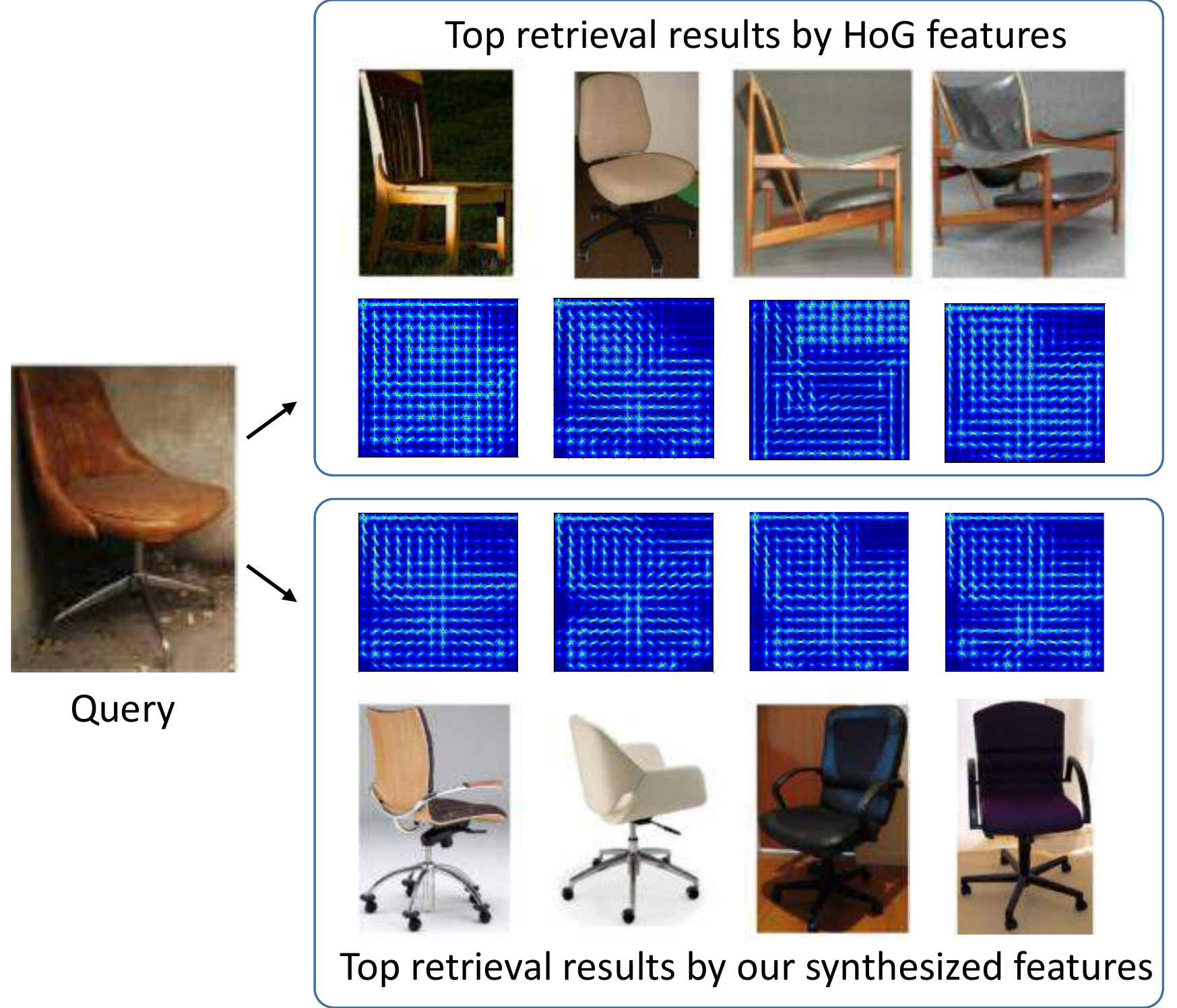}
\caption{{\bf View-invariant image retrieval.} \small Given each image, we synthesize its image features on a predefined list of viewpoints, and compare two images by their synthesized features from all viewpoints. Therefore, we can find images with similar objects regardless of viewpoint change. The first row shows the retrieval results by the HoG features on the original images only, and the second row visualizes the HoG features used for retrieval. The third row visualizes the synthesized HoG features, and the last row shows retrieval results using our synthesized features. Synthesized features are shown only on the view corresponding to the query image due to paper length constraint. }\label{fig:teaser}
\end{figure}

In the vision literature, substantial efforts have been made to achieve such invariant representations. Many kinds of view-invariant image features have been proposed for image comparison and recognition. The basic idea is to embed images into a \emph{common space}, so that the embedded point remains relatively stable as the viewpoint changes. The common space is usually built out of low-level image features with careful engineering~\cite{Lowe1999, bay2006surf, matas2004robust}. % [CITE common features such as SIFT].
Recently, there have also been successful efforts that learn such view-invariant features from large image datasets~\cite{krizhevsky2012imagenet,le2011learning}. However, these methods can only achieve invariance for small viewpoint changes, typically not more than $15\degree$. Another option is to choose the common space to be a high-level conceptual representation, based on object class or attributes~\cite{deng2011hierarchical,lampert2009learning}. However, much detailed geometric and physical information about the object gets lost in such embedding processes.
There has also been  work on reconstructing 3D geometry from an image~\cite{haosu_sig14}; however, these algorithms still lack the ability to recover detailed information and do not scale well.

In this paper, we choose the common space to be the \emph{space of 3D shapes}, similar to \cite{haosu_sig14}. Objects in images are indeed 2D projections of \emph{latent} 3D shapes; therefore, if we can obtain the latent 3D shape for each image, then images can be compared in a view-agnostic manner. Compared with low-level features, this approach achieves view-invariance for arbitrary viewpoint changes. Compared with high-level representations, 3D shape space is closer to the physical form of objects and therefore preserves more basic and detailed information.

Reconstructing 3D geometry, however, is a really challenging task and unnecessary in many cases, for example, image comparisons are typically based on image feature sets. In this paper {\em we focus on reconstructing the features of different views of the latent shape of an imaged object} (Fig.~\ref{fig:teaser}). So effectively we choose a multi-view representation of a shape, also known as 2-1/2D representation~\cite{Savarese_ICCV2007_Multiview,chen2003visual}, where each 3D shape is described by a set of images rendered from a predefined list of viewpoints. The descriptor of the latent shape is just a concatenation of image features from each of the views (see \S\ref{sec:notation}). In this way, the problem of reconstructing the features of the latent shape can be formulated as: given an object image (one view), reconstruct its features from other novel views (in the desired view set).

This problem is very challenging, because it is naturally ill-posed. The input is only one single real image based on one view --- thus information seems to be missing for reconstruction from novel views, even if all we seek is features for the new views.
Therefore, we introduce a 3D shape collection from the same object class as a non-parametric prior.
The intuition of our proposed method is that, given the novel view, we find related parts in the observed view that can best help us estimate the novel view. For this task, we have two guides: the image from the observed view, and the entire shape collection. Thus we explore two kinds of relationships to accomplish feature synthesis.

The first type of relationship reflects the \emph{intra-shape structure} which builds the relationships between the novel view and the observed view. More specifically, such relationships characterize the correlation of features at different locations of different views. Such correlations naturally exist because images from different views observe the same underlying 3D shape, whose parts may be further correlated by 3D symmetries, repetitions, and other factors. We use a probabilistic framework to quantitatively measure such correlations, aiming to estimate the  ``surrogate suitability'' of one image patch in one view to predict another patch in another view. Such relationships can be discovered efficiently and accurately from the shape collection.

The second type of relationship reflects the \emph{inter-shape structure} which builds the relationships between the image object and the shape collection. Although the entire shape space for our multi-view shape representation is highly nonlinear, local neighborhoods can be well approximated by a linear low-dimensional subspace~\cite{tenenbaum2000global}. This allows us to synthesize novel points in the shape space through linear interpolation, so as to approximate the latent image object. The key point for capturing this relationship is to estimate appropriate coefficients for the interpolation, and we use an approach derived from locally linear embedding (LLE) methods~\cite{roweis2000nonlinear}.

To summarize, for each patch in the novel view, the intra-shape relationships allows us to find which patches in the observed view are its best surrogates, and the inter-shape relationships teach us how the feature of the new patch should be synthesized from those of its surrogates. In this way we can populate with features for all views of the latent object in our image, effectively creating its representation in our shape space.

Our major contributions in this paper are:
\begin{itemize}	
	\item We propose a method for synthesizing object image features from unobserved views by exploiting inter-shape and intra-shape relationships;
	\item Given the synthesized image features for novel views, we are able to compare two images of the same or different objects by comparing their synthesized multi-view shape features. The resulting distance is view-invariant and achieves much better performance on fine-grained image retrieval and classification tasks when compared with previous methods.
%	\item A fundamental approach for extracting relationships across shapes, across views and across spatial patches.
\end{itemize}

%[Discovery of relationships] For each shape, we learn the interconnection of different rendered views at different spatial patchs. for the shape collection we learn the interconnection among different shapes. These interconnections can be viewed as constraints for the low-dimensional subspace mentioned earlier.

%\subsection{Pose Estimation}
%cite an existing work;
% \subsection{Global Method}
% manifold assumption;
% \subsection{Shape Representation: Light-field Descriptor}
% \subsection{Global Novel-view Feature Synthesis}

%% file: formulation.tex
\begin{figure*}[th!]
	\centering
	%	\includegraphics[width=\linewidth]{fig/overview.pdf}
	%	\caption{{\bf Method Overview.} {\small  Given an image, we want to synthesize the feature of its latent shape in a novel view. The synthesis is patch-by-patch. To predict the feature in the blue patch, we first look for a region in the observed view that are the most correlated to it, denoted as the surrogate region (purple patches). The surrogate region is found by analyzing the shape collection ({\bf Surrogate Region Discovery}, Sec~\ref{sec:surrogate_discovery}). Then, at the observed view, we reconstruct the surrogate region by a linear combination of the same region on the same view from all shapes in the shape collection ({\bf Estimation of Synthesis Parameter}, Sec~\ref{sec:parameter_estimation}). Finally, we transfer the linear combination coefficients back to the novel view, then the blue patch feature can be reconstructed by a linearly combination of the features at the same patch on the novel view of all shapes ({\bf Feature Synthesis}, Sec~\ref{sec:feature_synthesis})}.\label{fig:overview}}
\includegraphics[width=\linewidth]{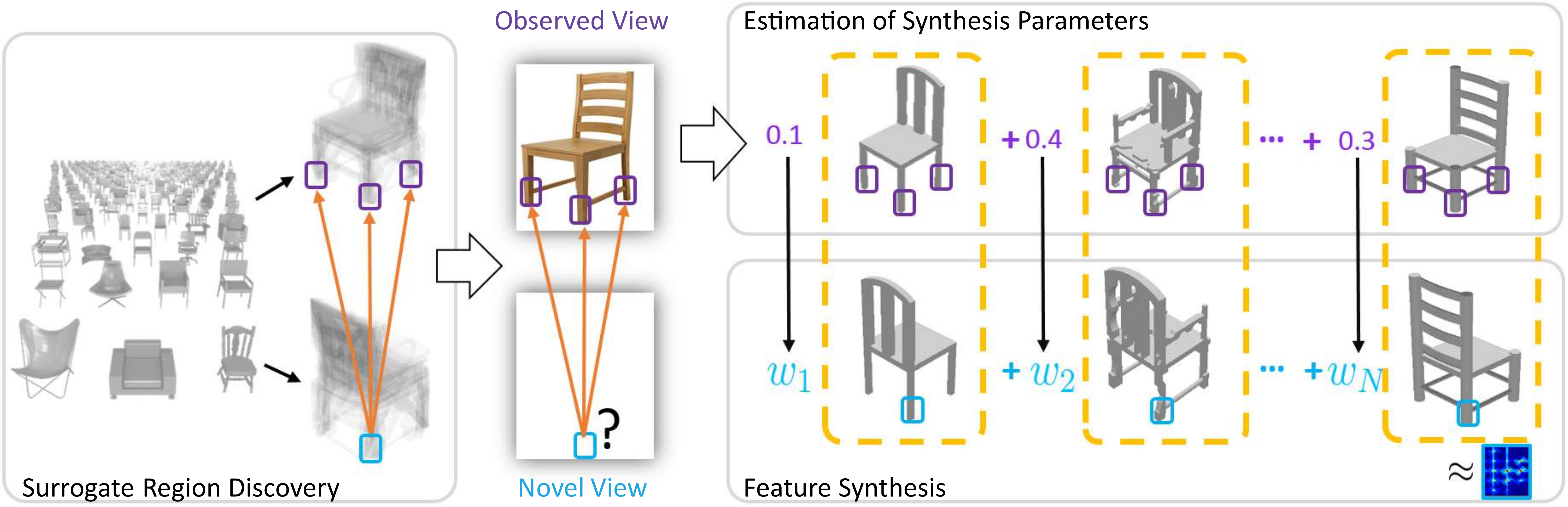}
\caption{{\bf Method overview.} {\small  Given an object image, we synthesize image features for novel views of the latent underlying object. The synthesis is done patch-by-patch. To predict the feature in the blue patch, we first look for regions in the observed view that are the most correlated to it, denoted as the surrogate regions (purple patches). The surrogate regions are found by analyzing the shape collection ({\bf Surrogate Region Discovery}, \S\ref{sec:surrogate_discovery}). Then, at the observed view, we learn how to reconstruct the surrogate region by a linear combination of the same region in the same view from all shapes in the shape collection ({\bf Estimation of Synthesis Parameter}, \S\ref{sec:parameter_estimation}). Finally, we transfer the linear combination coefficients back to the novel view to reconstruct the features in the blue patch by a linear combination of the features at the same patch on the novel view from all shapes ({\bf Feature Synthesis}, \S\ref{sec:feature_synthesis})}.\label{fig:overview}}
\end{figure*}
\label{sec:Formulation}
\mypara{Problem Input}
%Given an object cropped by a bounding box from a single image, the core task is to build a multi-view shape descriptor that composes image features at multiple views. As this problem is naturally ill-posed, we introduce a 3D shape network from the same object class as a non-parametric prior, assuming that the object category is given.

Our input contains two parts:

\noindent 1) an image of an object $O$ with bounding box and known class label. With recent advances in image detection and classification~\cite{DBLP:journals/corr/RussakovskyDSKSMHKKBBF14}, obtaining object label and bounding box has become much easier than before. All following steps are performed on a cropped image which only contains the object.

\noindent 2) a collection of 3D shapes (CAD models) from the same class. All 3D shapes are orientation aligned in the world coordinate system during a preprocessing step. Each shape is stored as a group of rendered images from the predefined list of viewpoints. Each rendered image is also cropped around the object. The view for object $O$ in the input image is estimated and approximated by one of the predefined viewpoints (\S\ref{sec:preprocessing}). To preserve detail information, the input image and the rendered images are resized to a fixed size and partitioned into overlapping square patches. Patch-based features such as HoG are extracted for each patch.

\mypara{Problem Output}

The output is the multi-view shape representation of the \emph{latent shape} of the input object, consisting of one image descriptor for each of the predefined views.

Without loss of generality, the key subproblem can be formulated as: given the object image in the input viewpoint $v_0$, estimate its features from another novel viewpoint $v_1$. The full multi-view representation can then be obtained by repeating this process for each predefined viewpoint.

%% file: overview.tex
\mypara{Method Overview}
\label{sec:Overview}

The proposed framework is shown in Fig.~\ref{fig:overview}.
For a specific patch in the novel view (the query patch), we seek to find the patches on the observed view which can best predict it (\emph{Surrogate Region Discovery} in Fig.~\ref{fig:overview}), and then learn how the features in those ``surrogate'' patches at the observed view can be best synthesized by the 3D model views (\emph{Estimation of Synthesis Parameters} in Fig.~\ref{fig:overview}). We finally apply the same synthesis method to the desired query patch (\emph{Feature Synthesis} in Fig.~\ref{fig:overview}). %In summary, our approach consists of three steps.
%(1) {\bf Surrogate region discovery}:  we find several patches at the observed view which are most correlated with the the query patch by analyzing corresponding views the shape collection. (2) {\bf Estimation of synthesis parameters}: at the observed view, we reconstruct the surrogate region by a linear combination of the same region on the same view from all shapes in the shape collection. The parameters are estimated by a locally-linear embedding (LLE) appoach. (3) {\bf Feature synthesis}: we \emph{transfer} the linear combination coefficients back to the novel view and reconstruct the query patch features by a linear combination of the features of the same patch on the corresponding views of all shapes.

%% file: approach.tex
\section{Novel View Image Feature Synthesis}
\label{sec:approach}
%As introduced in Sec~\ref{sec:Overview},

\subsection{Notation}
\label{sec:notation}

Our notation follows standard mathematical conventions. The set of preselected viewpoints is indexed by $\set{V}=\{1, \ldots, V\}$. Each rendered image or the input real image is covered by $G$ overlapping patches, indexed by $\set G=\{1, \ldots, G\}$. A patch-based feature set $\vec{f}=[\vec{x}_1^T; \ldots; \vec{x}_G^T] \in \mathbb{R}^{G \times D}$ is extracted for the image, where each $\vec{x}_g \in \mathbb{R}^D$ is a feature vector for patch $g$. So the multi-view shape descriptor is represented by a tensor $\vec{S}=[\vec{f}_1; \ldots; \vec{f}_V]\in\mathbb{R}^{V \times G \times D}$, in which each $\vec{f}_v$ is a feature of a rendered image at view $v$. Finally, the 3D shape collection is denoted by $\set{S}=\{\mat{S}_1, \ldots, \mat{S}_N\}$, where $\mat{S}_n$ denotes the multi-view descriptor of a shape $n$. For convenience, we further let $\vec{S}_{n, v,g}\in\mathbb{R}^D$ denote the features of the $g$-th patch in the $v$-th view of the $n$-th shape.

\subsection{Surrogate Region Discovery}
\label{sec:surrogate_discovery}
%{\color{blue}
%	Patch similarity across models are correlated;
%	
%	Provide the intuition by figures;
%	
%	Propose the formulation;
%}

\begin{figure}[h!]
	\includegraphics[width=\linewidth]{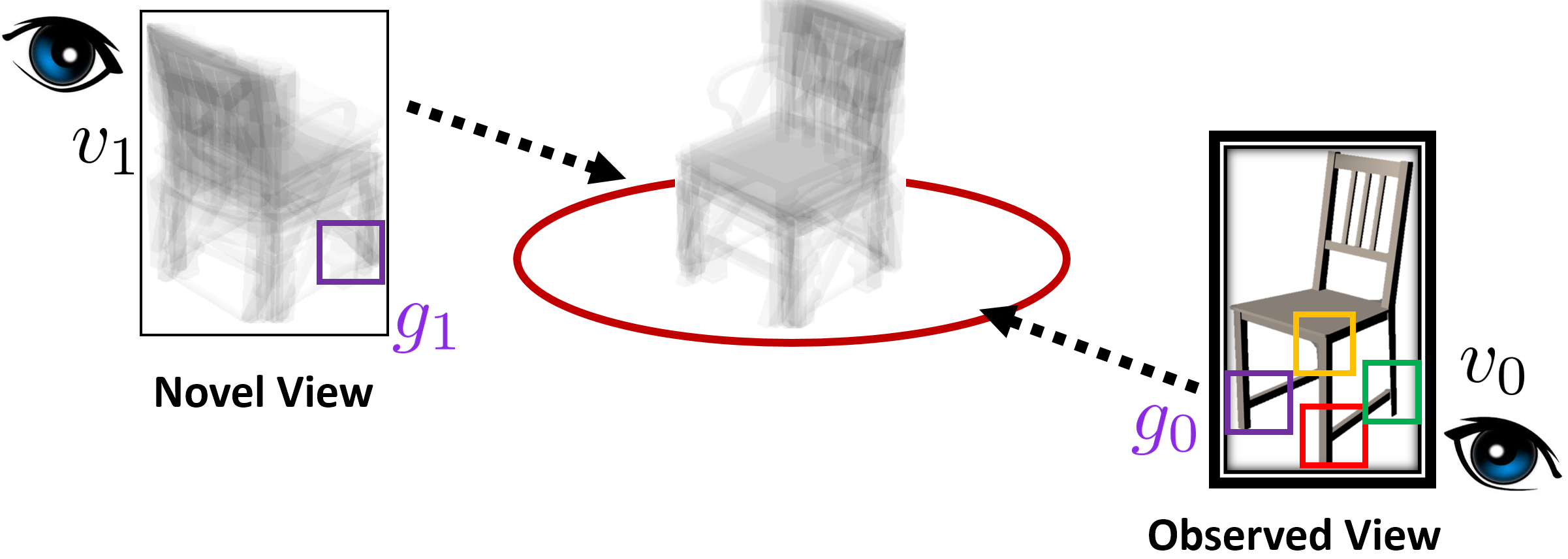}
	\caption{{\bf Patch surrogate relationship} (\S\ref{sec:surrogate_discovery})\label{fig:correlation}}
\end{figure}

Our ultimate goal is to transfer information across views, since we want to apply the synthesis parameter learned from one patch on one view to some other patches on other views. Therefore, we first exploit the cross-view patch appearance correlation as an important building block.

Fig.~\ref{fig:correlation} shows some intuitive examples about patch relationships. It is obvious that observing patch $g_0$ at view $v_0$ (purple box) helps us to determine the appearance of patch $g_1$ at view $v_1$ because they correspond to the same leg of a chair. Besides that, other factors such as symmetry and part membership in 3D shapes can also induce strong correlations among patch appearances. For example, the red patch in $v_0$ strongly correlates with $g_1$ because of chair symmetry; the green patch in $v_0$ is also correlated with $g_1$ because it belongs to the same part type as $g_1$ (chair leg). On the other hand, the appearance of chair back at $v_0$ will not be very helpful in determining $g_1$.

Therefore, there exists a group of patches at the observed view which are correlated with the query patch at the novel view, which we call {\em surrogate patches}; the region they form is called a surrogate region $\mathcal{R}$.

This relationship between patches across views can possibly be inferred by analyzing the shape geometry, but this is non-trivial and would require reliable object part segmentation, symmetry detection,etc. Therefore, we introduce an learning-based approach instead.

To precisely quantify such correlations between patches, we first introduce the concept of {\bf perfect patch surrogate}:
\begin{defn}
	Patch $g_0$ at view $v_0$ is a {\bf perfect patch surrogate} for patch $g_1$ at view $v_1$ if $\vec S_{i, v_0, g_0} = \vec S_{j, v_0, g_0}$ implies $\vec S_{i, v_1, g_1} = \vec S_{j, v_1, g_1}$ for {\bf any} shape pair $\vec S_i$ and $\vec S_j$.
	\label{defn:patch_surrogate}
\end{defn}
Intuitively, this means the similarity of patch $g_0$ at view $v_0$ implies the similarity of patch $g_1$ at view $v_1$ between a pair of 3D shapes.
Usually patches cannot be perfect surrogates for each other, so we seek for a probabilistic version of Definition~\ref{defn:patch_surrogate}:
\begin{defn}
	For a given patch $g_1$ at $v_1$, the {\bf surrogate suitability} of patch $g_0$ at view $v_0$ is defined as
    \begin{align*}
		\gamma(g_0; g_1) = \log P(\vec S_{i, v_1, g_1}=\vec S_{j, v_1, g_1}|\vec S_{i, v_0, g_0}=\vec S_{j, v_0, g_0}),
    \end{align*}
\end{defn}
$\gamma(g_0; g_1)$ is a measure of how suitable patch $g_0$ is as a surrogate for patch $g_1$. Intuitively, larger $\gamma(g_0; g_1)$ indicates a stronger correlation (Fig.~\ref{fig:surrogatability}). Therefore, the surrogate region $\mathcal{R}(g_1)$ can either consist of the top $k_p$ patches with highest $\gamma(g_0;g_1)$, or $\mathcal{R}(g_1) = \{g_0: \gamma(g_0;g_1)>\tau\}$, where $k_p$ or $\tau$ is determined empirically. We discuss the estimation of $\gamma(g_0, g_1)$ in the next section.

\begin{figure}[h!]
\includegraphics[width=\linewidth]{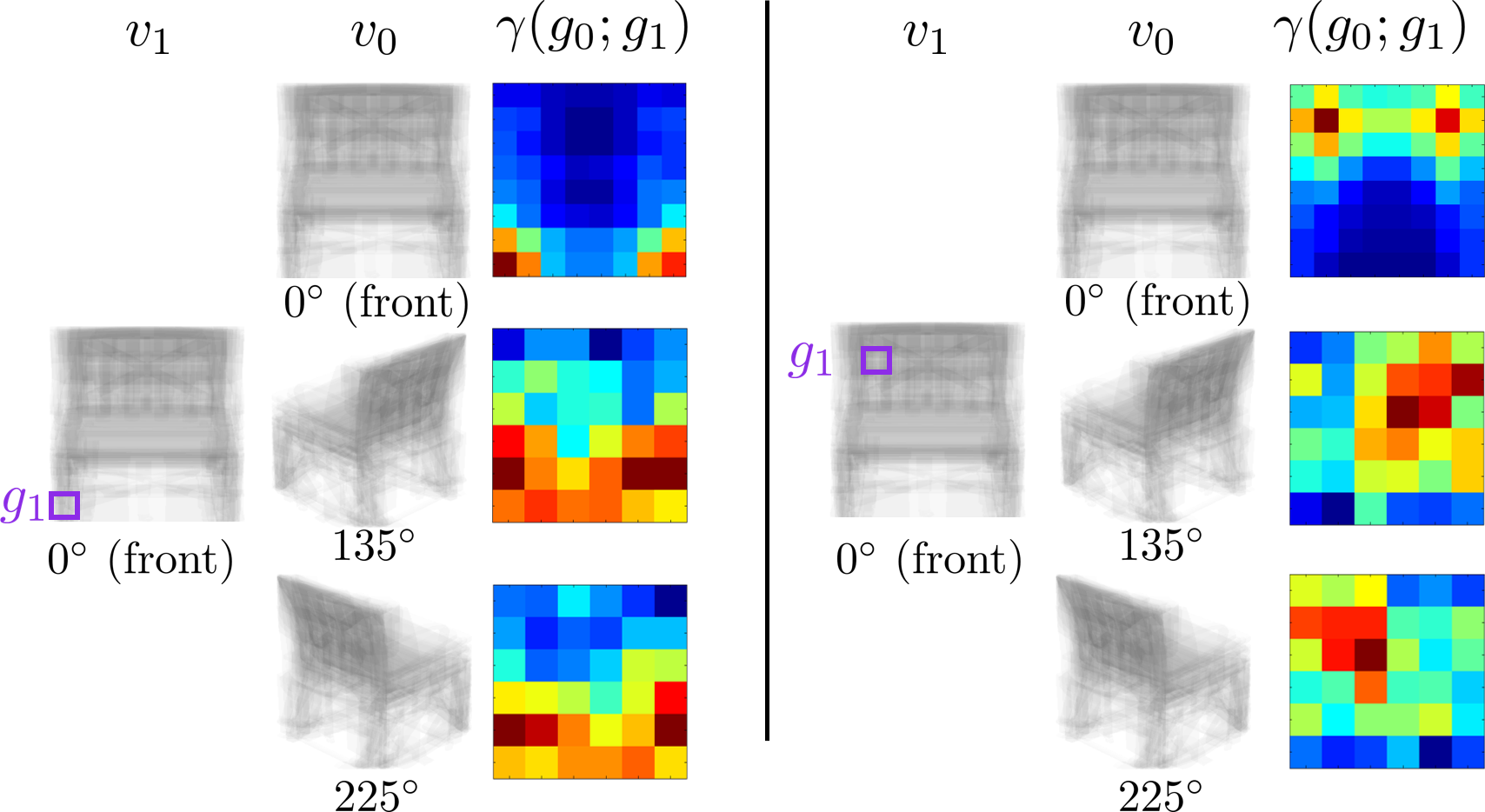}
\caption{{\bf Visualization of patch surrogate suitability.} {\small Two examples of the surrogate suitability from $g_1$ in $v_1$ to patches in view $v_0$. Red means large $\gamma$. For example, in the left figure, $g_1$ corresponds to the tip of right-front leg at $v_1$ (front view). At the front view itself, the left-front and right-front leg tips have higher surrogate suitability for $g_1$ because of symmetry; at the $225\degree$ view, the left-back, right-back and right-front leg tips have higher surrogate suitability because of symmetry and part membership.} \label{fig:surrogatability}}
\end{figure}

\subsubsection{Estimation of Patch Surrogate Suitability}
\label{ssec:estimation}

With the large-scale shape collection at hand, we adopt a learning based approach to estimate the (probabilistic) patch surrogate suitability in a data-driven manner.

Estimating $\gamma(g_0; g_1)$ is a non-parametric density estimation problem. As image features are high-dimensional continuous variables, theoretical results indicate that the sample complexity for reliable estimation is very high and infeasible in practice. To overcome the difficulty, we quantize features into a vocabulary $\mathcal{D}$ containing $D$ visual words. For notation convenience, we denote the codeword of $\vec S_{i, v_0, g_0}$ by $A^i_{g_0}$ and $\vec S_{i, v_1, g_1}$ by $A^i_{g_1}$, then
\begin{align}
\gamma(g_0; g_1) = \log P(A^i_{g_1}=A^j_{g_1}|A^i_{g_0}=A^j_{g_0})\label{eq:surrogatability}
\end{align}
where $P$ is the probability measure.

Estimating \eqref{eq:surrogatability} by an empirical conditional distribution still requires a large amount of samples. However, we show that \eqref{eq:surrogatability} can be cast as a R\'enyi entropy estimation problem.
%As a corollary of results in \cite{DBLP:journals/corr/AcharyaOST14},
We can prove that the \emph{optimal} sample complexity needed for estimating \eqref{eq:surrogatability} is $\Theta(D)$ (Theorem 1 in Appendix). Roughly speaking, with $N=\Theta(D)$ shapes, we can accurately estimate \eqref{eq:surrogatability} with high probability. The proof also suggests an algorithm to estimate Eq~\eqref{eq:surrogatability} as below:
\begin{align*}
&\hat{\gamma}(g_0; g_1) = \log \sum_{\substack{(A_{g_0}, A_{g_1})\\\in\mathcal{D}\times \mathcal{D}}}\hat{P}^2(A_{g_0}, A_{g_1}) -\log \sum_{A_{g_1}\in\mathcal{D}} \hat{P}^2(A_{g_1})\,.%\\
% = &\log P(A^i_{g_1}=A^j_{g_1}, A^i_{g_0}=A^j_{g_0})-\log P(A^i_{g_0}=A^j_{g_0}) \nonumber
\end{align*}
Here, probabilities $P^2(x)$ should be estimated by $\hat{P}^2(x)=\frac{N_x(N_x-1)}{N^2}$, where $N_x$ is the total number of times value $x$ appears in samples and $N=\sum_x N_x$.

\subsection{Estimation of Synthesis Parameters}
\label{sec:parameter_estimation}
%{\color{blue}
%	For each region in the novel view, we hypothesize that its feature can be reconstructed by the linear combination of the same region of other shape on the same view.
%
%	The linear coefficient was estimated by its correlated part on the observed view.
%	
%	
%	Show the illustration;
%	
%	Propose the formulation;
%}
% \begin{figure}[h!]
%	\includegraphics[width=\linewidth]{fig/placeholder.pdf}
%	\caption{{\bf Locally Linear Reconstruction of Patch Feature} \label{fig:linear_reconstruction}}
% \end{figure}

As we have mentioned, the global shape space for our multi-view representation is non-linear and high-dimensional. Our assumption, however, is that shapes in a local neighborhood can be well approximated by a locally linear and low-dimensional subspace. Since the multi-view representation is actually a concatenation of features from all patches of all views, this local linearity does not only hold for the whole shape, but it also holds for each view of the shape, for each patch of the view, or even for a subset of patches of the view. In other words, features for the patches from the same location(s) on the same view of all shapes also lie in a locally linear subspace.

For any patch $g$ in view $v$, its feature is denoted as $\vec x_{v,g}\in \mathbb{R}^D$. We use $\mat{S}_{:, v, g}\in\mathbb{R}^{D\times N}$ to denote the feature matrix collecting patch $g$ of view $v$ of all 3D models, then local linearity tells us that
\begin{align}
\vec x_{v, g} \approx \mat{S}_{:, v, g} \vec w_{v, g}\,,
\end{align}
\
where $\vec w_{v, g} \in \mathbb{R}^N$ is the reconstruction coefficient.

Given a surrogate region $\mathcal{R}$ on the observed view, its features should be a linear combination of the same region across different 3D shapes. So $\vec w_{v_0, \mathcal{R}}$ can be estimated by solving an Locally Linear Embedding (LLE) problem:
\begin{equation}
\begin{aligned}
& \underset{\vec w_{v_0, \mathcal{R}}}{\text{minimize}} & & \sum_{g_0 \in \mathcal{R}} \|\vec x_{v_0, g_0}-\mat{S}_{\mathcal{N}, v_0, g_0} \vec w_{v_0, \mathcal{R}}\|^2 \,,\\
& \text{subject to} & & \vec w_{v_0,\mathcal{R}} \ge \vec 0;
\quad \vec w_{v_0, \mathcal{R}}^T \vec 1 = 1\,, \label{eq:LLE}\\
%& & & \mathcal{N}=\{n: \sum_{g_0 \in \mathcal{R}} \|\vec x_{v_0, g_0}-\mat{S}_{n, v_0, g_0}\|^2<\epsilon\}.
\end{aligned}
\end{equation}
where $\mathcal{N}$ denotes the $k$-nearest shapes obtained from the whole shape collection by comparing the rendered images on $v_0$ with the input image, thus $\mat{S}_{\mathcal{N}, v_0, g_0}\in\mathbb{R}^{D\times k}$ and ${\vec w}_{v_0, \mathcal{R}}\in\mathbb{R}^k$.

Note that our reconstruction coefficient $\vec{w}_{v_0, \mathcal{R}}$ is specific to the choice of view $v_0$ and patch(es) $\mathcal{R}$, unlike previous locally linear reconstruction methods assuming uniform $\vec w$ for the whole image descriptor~\cite{vidal2010tutorial}. Experiments show that spatial-varying coefficients allow us to recover features more precisely and with better spatial locality as compared with image-wide uniform coefficients (Fig.~\ref{fig:nPatch} and Fig.~\ref{fig:threshold}). % Note that, previous locally linear reconstruction methods such as [citation] generally assume uniform reconstruction coefficient for the whole image descriptor, whereas our reconstruction coefficient $\vec{w}_{v_1, g}$ is specific to the choice of patch.

\subsection{Feature Synthesis}
\label{sec:feature_synthesis}
%{\color{blue}
%	
%	Show the illustration;
%	
%	Propose the formulation;
%	
%	Justification about weight transferability;
%}
Now that we have the synthesis coefficients estimated for $\mathcal{R}$ on view $v_0$, we have to decide how to transfer it back to $v_1$, so that we can synthesize $\vec{x}_{v_1, g_1}$ by  apply the coefficients on features of $g_1$ on $v_1$ from all shapes.

We make the following assumption to connect the weight across views: for a shape $\vec S$, if a patch $g_0$ can surrogate $g_1$ very well (with high $\gamma(g_0;g_1)$), and $\vec x_{v_0, g_0} \approx \mat{S}_{:, v_0, g_0} \vec w_{v_0, g_0}$ and $\vec x_{v_1, g_1} \approx \mat{S}_{:, v_1, g_1} \vec w_{v_1, g_1}$, then their weights are the same, i.e. $\vec w_{v_0, g_0} \equiv \vec w_{v_1, g_1}$.

\begin{figure}[h!]
\centering
	\includegraphics[width=0.8\linewidth]{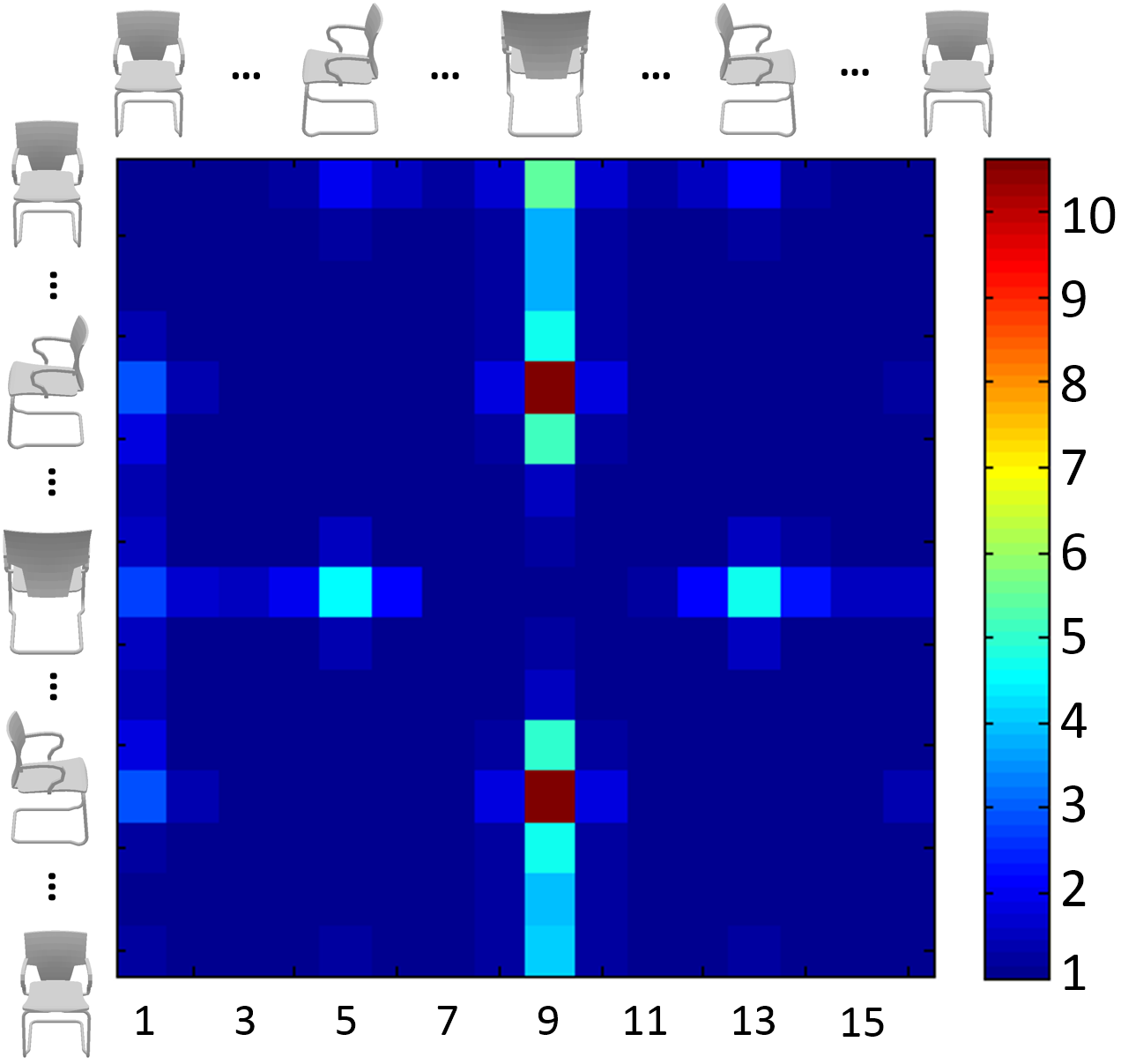}
	\vspace{-1mm}
	\caption{{\bf \small Evaluation of weight transferability.} \small The experiments are performed on rendered images of the shape collection (see \S\ref{sec:analysis} for details of experiment setup). The $(i,j)$-th element is obtained as below: we estimate the coefficients on view $v_i$ for all other shapes to reconstruct $S_{0,v_i}$, and apply them to $v_j$ to obtain the reconstructed feature $\hat{S}_{0,v_j}$; then we check the ranking of the distance between ground truth $S_{0,v_j}$ and reconstructed $\hat{S}_{0,v_j}$ in the list of distance values between $S_{0,v_j}$ and all other shapes on $v_j$. $(i,j)$-th element shows the ranking averaged through the whole collection. If the ranking is 1, it means the reconstructed feature is the closest to the ground truth feature, indicating that the reconstruction is good enough to describe $S_0$ on $v_j$. } \label{fig:transferrability}
\end{figure}

Empirical verification of this assumption is shown in Fig.~\ref{fig:transferrability}. The $(i,j)$-th element in the matrix shows the transferability from view $v_i$ to $v_j$. It measures how close the synthesized feature on view $v_j$ is to its ground truth version when using coefficients estimated on $v_i$. Each entry could range from 1 to the size of model collection (5,057 in this experiment). The closer the value is to 1, the better the transferability is between $v_i$ and $v_j$. The average value of the whole matrix is 1.39, quite close to 1, meaning that the weights transferred across views can reconstruct the features very well. Note that there are some entries indicating bad transferability between specific views. For example, view 5 and 9, which are the side view and back view respectively, cannot be transferred to each other very well because they share less common information. 

Therefore, $\vec w_{v_1, g_1}$ can be replaced by $\vec w_{v_0, \mathcal{R}}$ if $\mathcal{R}$ is the appropriate surrogate region on $v_0$ for $g_1$. We can reconstruct the feature by $\vec x_{v_1,g_1} = \mat{S}_{\mathcal{N}, v_1, g_1} \vec w_{v_0, \mathcal{R}}$. Fig.~\ref{fig:synthesis_visualization} shows two examples of our synthesized image features.

\begin{figure}[h!]
\centering
	\includegraphics[width=\linewidth]{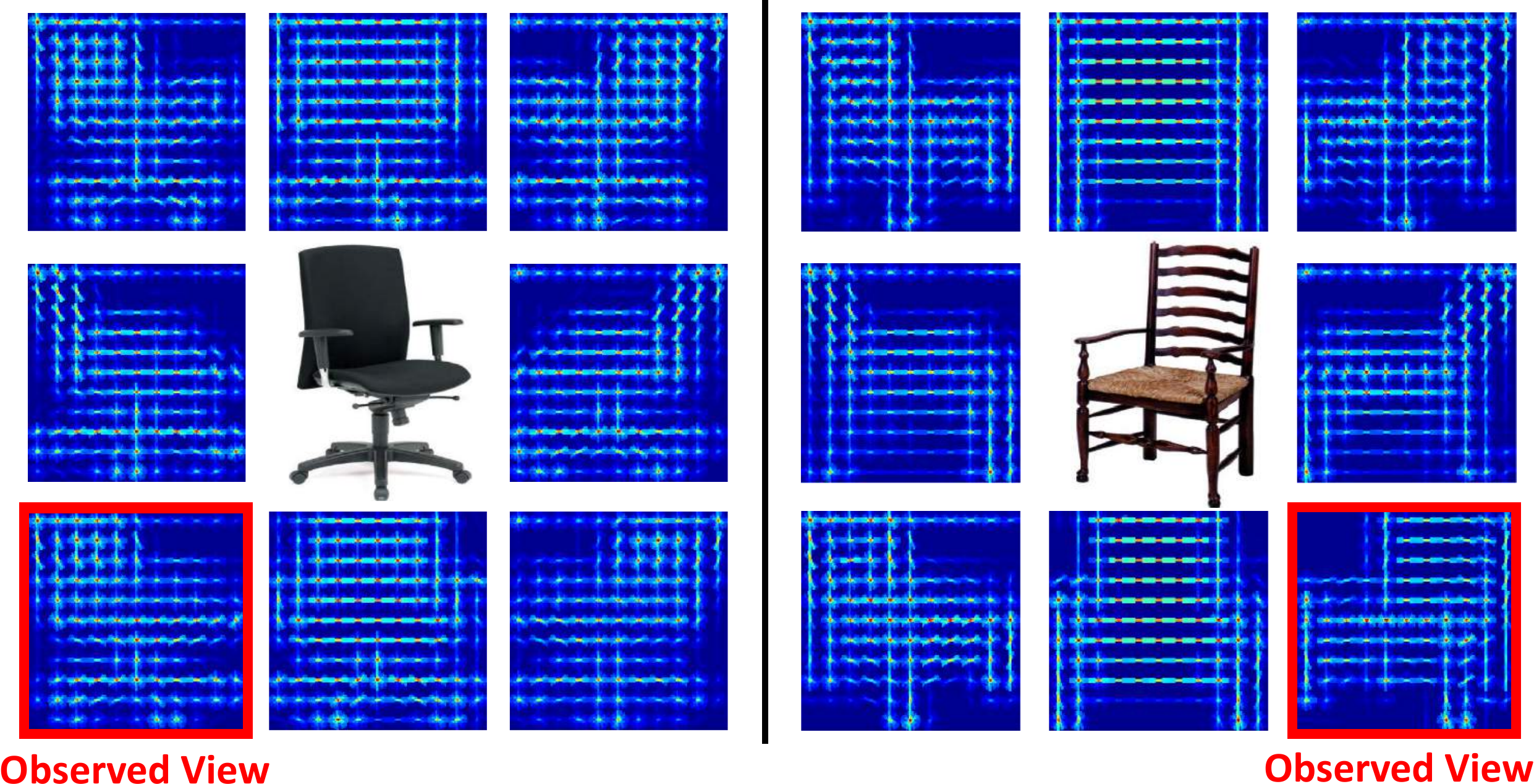}
\caption{{\bf Visualization of synthesized HoG features on 8 canonical views} On observed view, the original HoG feature is shown. \label{fig:synthesis_visualization}}
\end{figure}

%\begin{figure*}[t!]
%	\includegraphics[width=\linewidth]{fig/synchronization.eps}
%	\caption{{\bf Correlation} \label{fig:synchronization}}
%\end{figure*}

%\subsection{Algorithm Summary}
%\begin{figure}[h!]
%	\includegraphics[width=\linewidth]{fig/placeholder.pdf}
%	\caption{{\bf Algorithm} \label{fig:algorithm}}
%\end{figure}

%% file: experiment.tex
\subsection{Data Preparation}
\label{sec:preprocessing}
%{\color{blue}
%	Shape collection preprocessing: joint shape alignment; shape rendering.
%	
%	Image preprocessing: pose estimation.
%}

\mypara{Large-scale Shape Dataset}
We introduce a large-scale shape collection containing human-built 3D meshes from 100 man-made object categories. There are $\sim$100,000 3D models in total. The number of models per class varies from 20 (purse) to over 13000 (table). The models of our dataset are from the Trimble 3D Warehouse and the annotations are crowd-sourced from Amazon Mechanical Turk (AMT). Each object category in our shape dataset is further mapped to a synset in ImageNet~\cite{Deng2009}. Please refer to the supplementary material for more details on the dataset.

\mypara{Shape Collection Preprocessing}
To align the input 3D models, we employ the method described in~\cite{Huang:2013:FSL}, which jointly optimizes the orientation of all input 3D models to minimize the sum of distances between corresponding points computed using pair-wise alignment. To render 3D models, we sample a pre-specified number of different viewpoints over the viewing sphere centered at the shape.

The default setting for underlying features is as below unless specified otherwise: each rendered image is resized to $112\times 112$ and partitioned into $32\times 32$ patches which overlap with each other by 16 pixels, forming $6\times 6$ patches in total; HoG features are extracted for each patch.

\mypara{Image Preprocessing} As noted in \S\ref{sec:Formulation}, we assume object bounding boxes are provided. Camera pose is estimated for the cropped object. We train a random forest using the rendered images of our aligned 3D models.
% (\cite{stark2013PAMI,gu2010discriminative,gall2011hough})
Image features are extracted under default setting as before.

\subsection{View-Agnostic Distance}

Given the synthesized image features on a list of novel views, the distance between two images is obtained as the L2 distance between the two aligned and concatenated vectors of their synthesized multi-view features. Since the viewpoint information has been factored out, the resulting distance is view-invariant. This distance is denoted as view-agnostic distance (VAD) in the following experiments.

\subsection{Method Analysis}
\label{sec:analysis}

In this section, we analyze the performance of our method under different settings and provide a thorough understanding. The quantitative analysis is done in a fine-grained image retrieval application. We take the class ``Chair'' as an example. We collect all images with bounding boxes in ``Chair'' and its fine-grained sub-categories from ImageNet, and verify their labels by AMT. There are in total 5,813 images in 15 fine-grained categories, denoted as a ``cluttered'' set. In contrast, a subset of 1,309 images with simple background is selected to form a ``clean'' set to help us better understand the performance. The ``Chair'' shape collection contains 5,057 models, rendered in 16 views.

In our experiment, each image is taken as query and all other images are sorted according to their distance to the query image. Images belonging to the same fine-grained category are regarded as correct. If the query image belongs to multiple fine-grain categories, returned images belonging to any of its categories are regarded as correct. Precision-recall curves are generated to evaluate the retrieval performance. Our proposed VAD is compared with L2 distance of the baseline HoG descriptors.
\begin{figure}[h!]
	\centering
	\begin{subfigure}[b]{0.48\linewidth}
		\includegraphics[width=\textwidth,height=5.1cm]{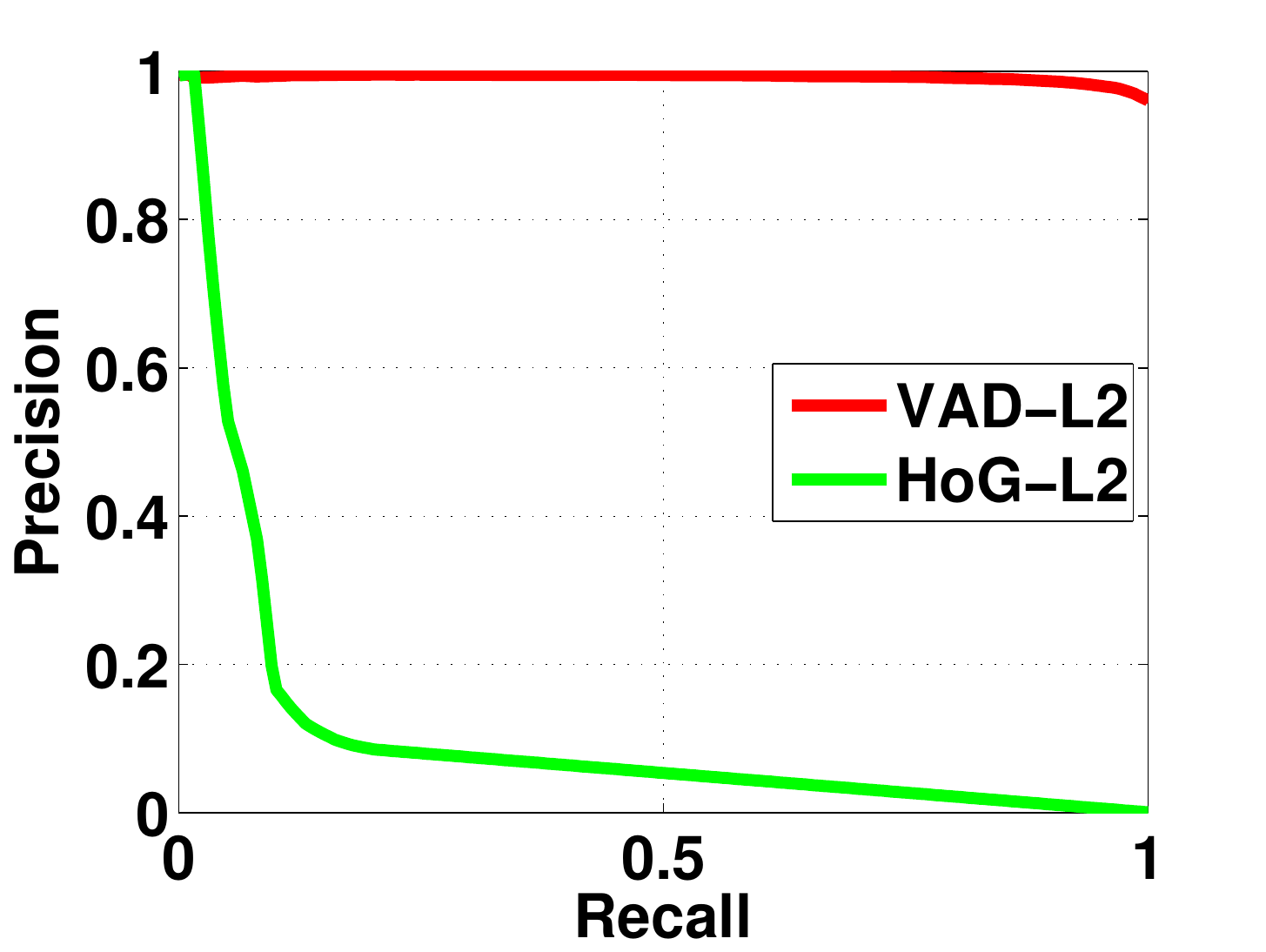}
		\caption{Fully controlled setting.}
		\label{fig:fullycontrol}
	\end{subfigure}%
	\begin{subfigure}[b]{0.48\linewidth}
		\includegraphics[width=\textwidth]{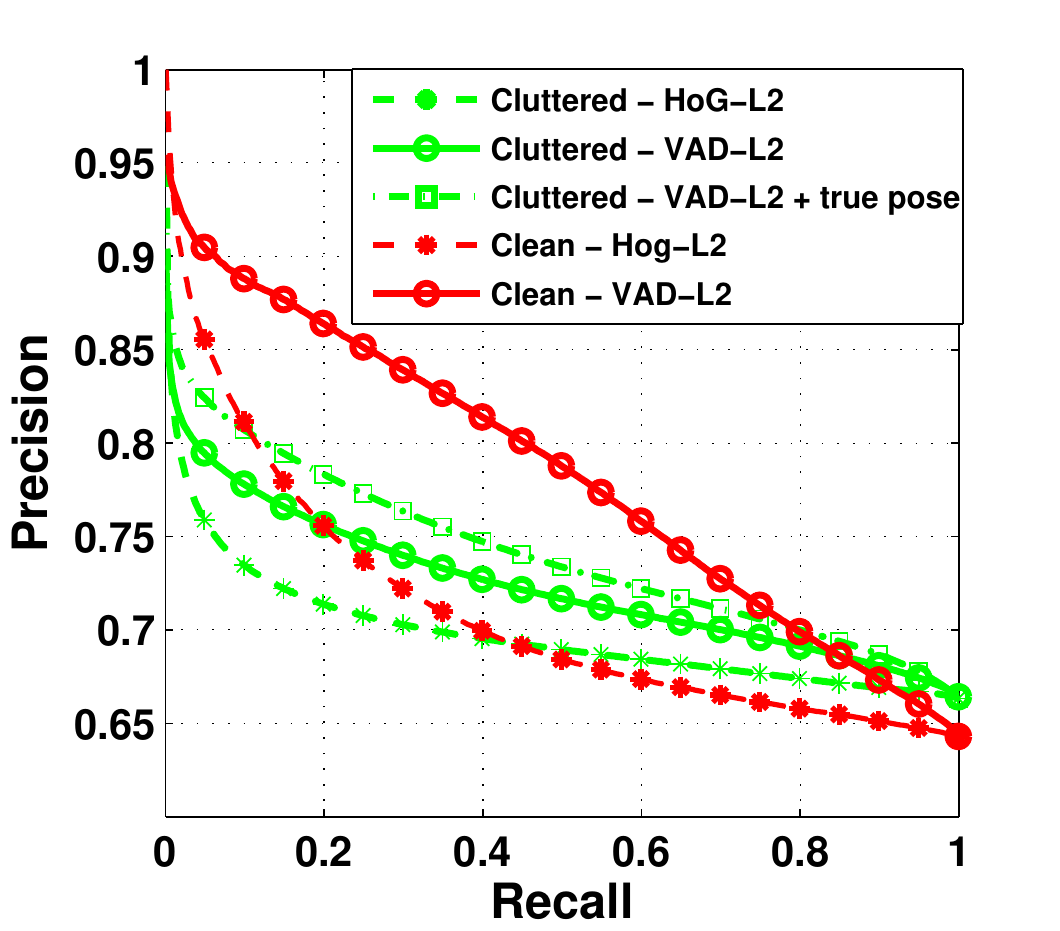}
		\caption{Clean vs. cluttered}
		\label{fig:clean_clutter}
	\end{subfigure}
	\caption{\bf Performance comparison.}\label{fig:comparison}
\end{figure}

\mypara{Fully Controlled Setting}

We first run an experiment using the data set from \cite{Aubry14}, which consists of 1393 3D chair models rendered in 62 viewpoints in an almost photo-realistic manner. We take all $1393\times62 = 86,366$ images and perform image retrieval with each image as query. Other view images from the same model are considered correct. This experiment is under fully controlled setting: the background is absolutely clean; pose estimation is 100\% accurate, and there is even no pose discretization error, i.e. the estimated pose and the true pose of the imaged object are exactly the same. This setting is the most favorable for our proposed method; therefore, we achieve almost perfect performance as shown in Fig.~\ref{fig:fullycontrol}. Since the chair models in \cite{Aubry14} and our models are from the same source, the chair models used here have been screened to avoid overlap with the 1,393 chairs, So it is guaranteed that there is no exact model for any of the query images.

%\begin{figure}[h!]
%\centering
%\includegraphics[width=0.75\linewidth]{fig/seeing3dchair}
%\caption{{\bf Performance Comparison in Fully Controlled Setting.} } \label{fig:fullycontrol}
%\end{figure}

\mypara{Clean background v.s. Cluttered background}

Fig.~\ref{fig:clean_clutter} shows the precision-recall curves of HoG-L2 distance and the proposed VAD. In both the ``clean'' (red lines) and ``clutter'' (green lines) cases, VAD has greatly boosted the performance of baseline HoG-L2. However, in ``clean'' case, pose estimation and nearest 3D model matching are more accurate, thus the performance boost is more significant compared with ``cluttered'' case. With better pose estimation algorithm, our VAD still has space to improve. As shown in the green dash-dot line in Fig.~\ref{fig:clean_clutter}, we can boost the performance even more in cluttered case if the ground truth viewpoint of the input image is given.

%\begin{figure}[h!]
%\centering
%\includegraphics[width=0.75\linewidth]{fig/clean_clutter_5plots}
%\caption{{\bf Performance Comparison.}} \label{fig:clean_clutter}
%\end{figure}

\begin{figure}[h!]
	\centering
	\begin{subfigure}[b]{0.48\linewidth}
		\includegraphics[width=\textwidth]{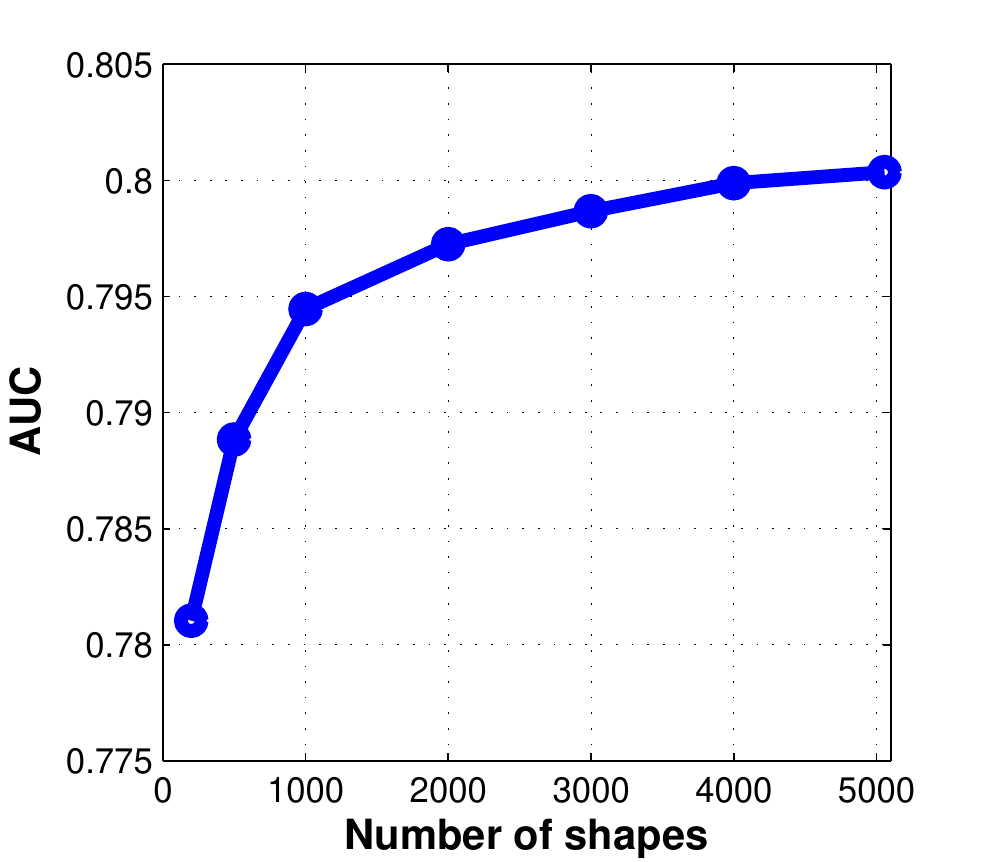}
		\caption{Size of shape collection.}
		\label{fig:nShapes}
	\end{subfigure}%
	\begin{subfigure}[b]{0.48\linewidth}
		\includegraphics[width=\textwidth]{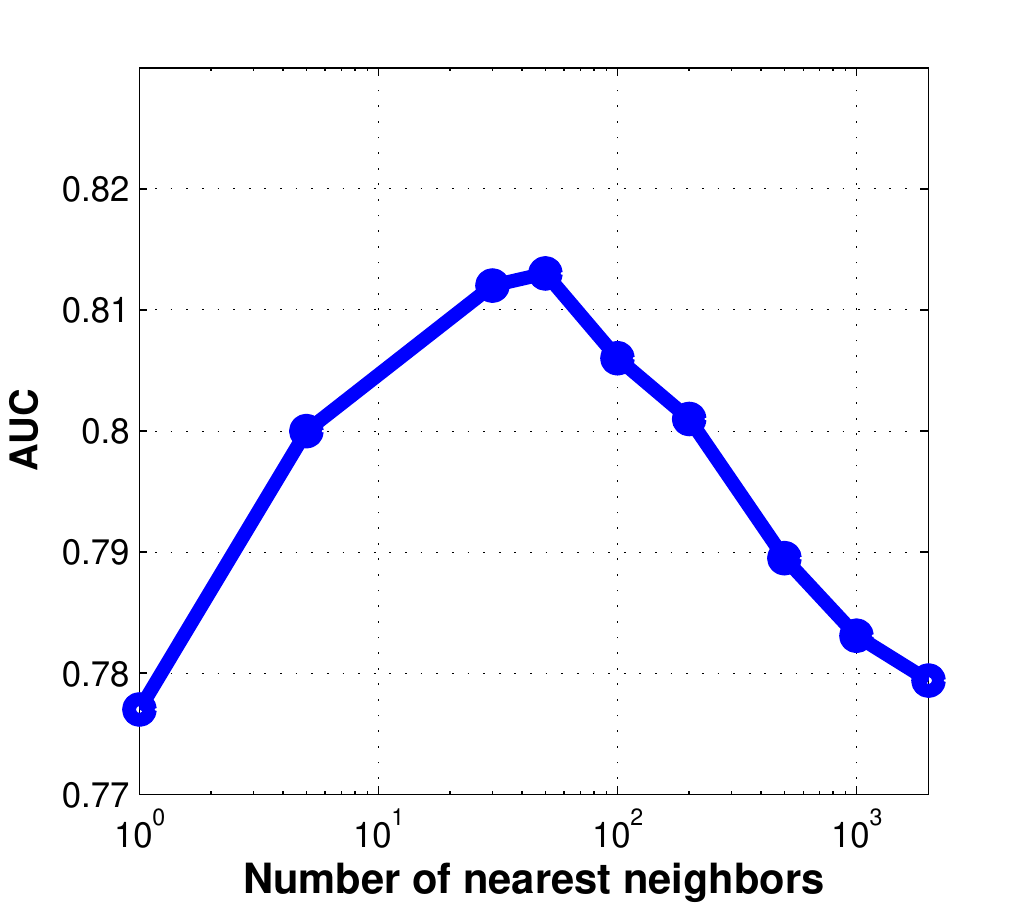}
		\caption{Neighborhood for LLE}
		\label{fig:nNeighbors}
	\end{subfigure}
	\begin{subfigure}[b]{0.48\linewidth}
		\includegraphics[width=\textwidth]{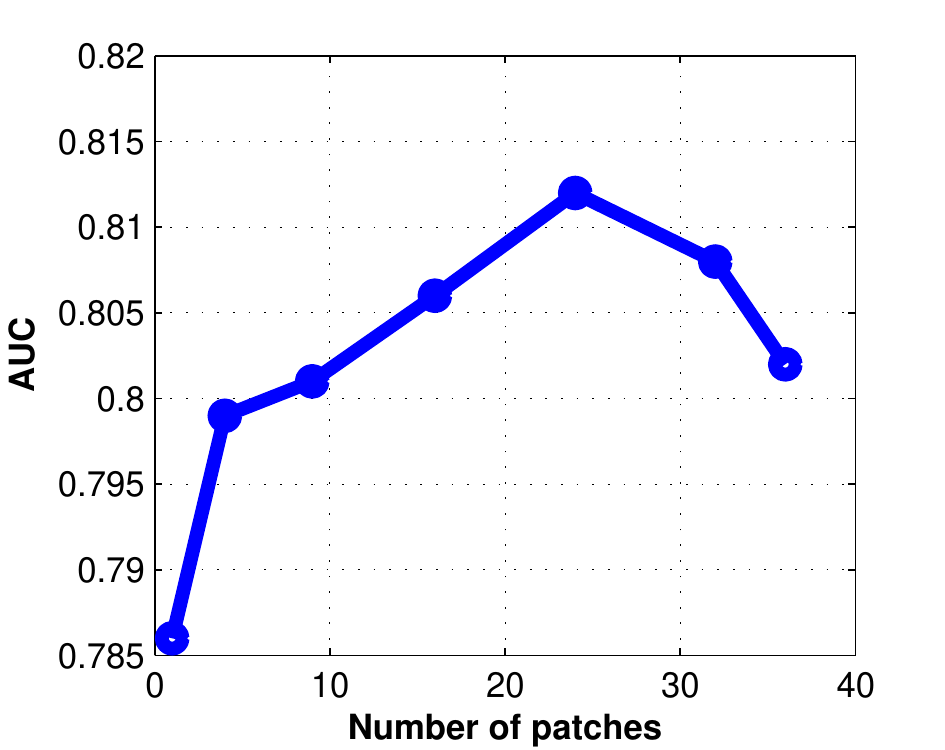}
		\caption{Locality by \texttt{\#} patches}
		\label{fig:nPatch}
	\end{subfigure}
	\begin{subfigure}[b]{0.48\linewidth}
		\includegraphics[width=\textwidth]{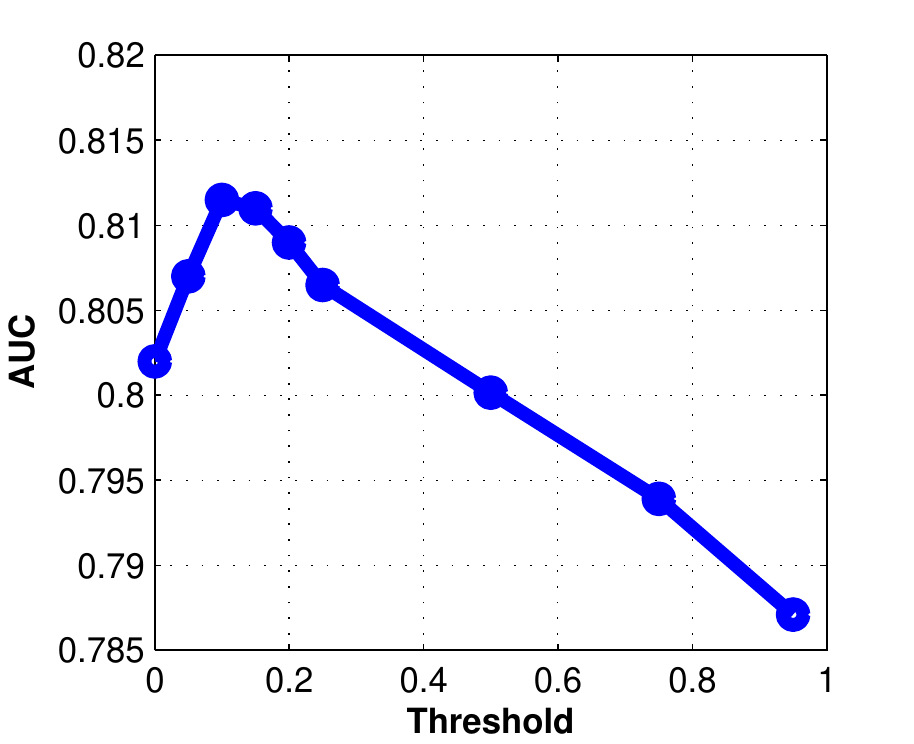}
		\caption{Locality by threshold}
		\label{fig:threshold}
	\end{subfigure}
%	\vspace{-3mm}
	\caption{\bf Parameter sensitivity.}\label{fig:sensitivity}
%	\vspace{-3mm}
\end{figure}
\mypara{Parameter Sensitivity Analysis}

Fig.~\ref{fig:nShapes} shows the performance changing with different number of 3D models. The performance for each experiment is summarized by the area under the precision-recall curve (AUC). The intuitive explanation is that, a larger shape collection is preferred since it can provide better coverage of the shape space and further help better reconstruct the descriptor on novel views. However, we also observe that the performance with 200 3D models is only 2\% lower than the performance with the full collection of 5,057 3D models. The reason is that our model has the ability to ``interpolate'' in the shape space, which compensates for the absence of large shape collection at query time. %we have an accurate estimation of patch surrogate relationships (\S\ref{ssec:estimation}), which was obtained using a large shape collection, thus it can compensate for the absence of large shape collection at query time.
We use the whole shape collection below.

Fig.~\ref{fig:nNeighbors} shows the performance changing with the parameter $k$ for obtaining the local neighborhood in Eq~\eqref{eq:LLE}. Specifically, for $k=1$, it is equivalent to using the most similar shape to represent the query object. It is beneficial to use an appropriate range of neighborhood to find reconstruction coefficients of the query latent shape, which is shown in Fig.~\ref{fig:nNeighbors}. $k=200$ is used for the rest experiments.

When finding the surrogate patch region on the observed view, the scope of locality can be defined by the number of patches $k_p$, or the threshold of the correlation score $\tau$(\S\ref{sec:surrogate_discovery}). We show the performance changing with $k_p$ or $\tau$ in Fig.~\ref{fig:nPatch} and Fig.~\ref{fig:threshold}, respectively. $k_p =  36$ in Fig.~\ref{fig:nPatch} (the last data point) or $\tau =  0$ in Fig.~\ref{fig:threshold} (the first data point) corresponds to the case when the whole image is selected as surrogate region, i.e. LLE coefficients are uniform across the whole image. We can see that, an optimal region exists as a surrogate for a query patch. Including more surrogate patches can increase the samples for linear coefficients estimation, hence achieves better robustness; but including too many patches will eventually bring in unrelated patches, which are harmful.

We also investigate performance changing with $k_p$ in different partitioning of image patches (Table~\ref{table:partition}). Similar conclusion can be made that the surrogate region cannot be too large nor too small. $k_p=9$ for $6\times 6$ is used as default.

\begin{table}
	\centering
	\begin{tabular}{|c|c|c|c|c|c|}
		\hline      &  single & 25\%  & 50\%  & 75\%  & 100\% \\
		\hline  6x6 & 0.786   & 0.801 & 0.806 & \textbf{0.81}  & 0.808 \\
		\hline  8x8 & 0.774   & 0.79  & 0.795 & \textbf{0.797} & 0.788 \\
		\hline  10x10 & 0.76  & 0.784 & \textbf{0.787} & 0.785 & 0.774 \\
%		\hline  12x12 & 0.69  & 0.762 & \textbf{0.778} & 0.773 & 0.762 \\
		\hline
	\end{tabular}
	\vspace{-2mm}
	\caption{{\bf AUC for different patch partition and different number of patches for surrogate region.}}
	\vspace{-5mm}
\label{table:partition}
\end{table}

\mypara{Applicability for CNN}

Our approach is not restricted to any specific kind of descriptors. Features extracted by convolutional neural networks (CaffeNet~\cite{jia2014caffe}) from different layers are tested here to replace the HoG features. The performance is shown in Table~\ref{table:otherfeature}.  It can be seen that for different choices of underlying features, our method can always boost the performance. %Therefore, our approach is robust to the selection of underlying features.
\begin{table}[h]
	\centering
\begin{tabular}{|l|c|c|}
	\hline               & Vanilla L2 & VAD (ours) \\
%	\hline  GIST         &  &  \\
	\hline  CNN (CaffeNet, pool 1) & 0.662 & \textbf{0.677} \\
	\hline  CNN (CaffeNet, pool 2) & 0.697 & \textbf{0.745} \\
	\hline  CNN (CaffeNet, pool 5) & 0.690 & \textbf{0.746} \\
	\hline  CNN (CaffeNet, fc 6) & 0.748 & \textbf{0.788} \\ 	
	\hline  CNN (CaffeNet, fc 7) & 0.744 & \textbf{0.785} \\ 	
	\hline
\end{tabular}
%\vspace{-2mm}
\caption{\bf Performance by other kinds of image features.}
\label{table:otherfeature}
%\vspace{-5mm}
\end{table}

\begin{figure*}[t!]
\centering
	\includegraphics[width=\linewidth]{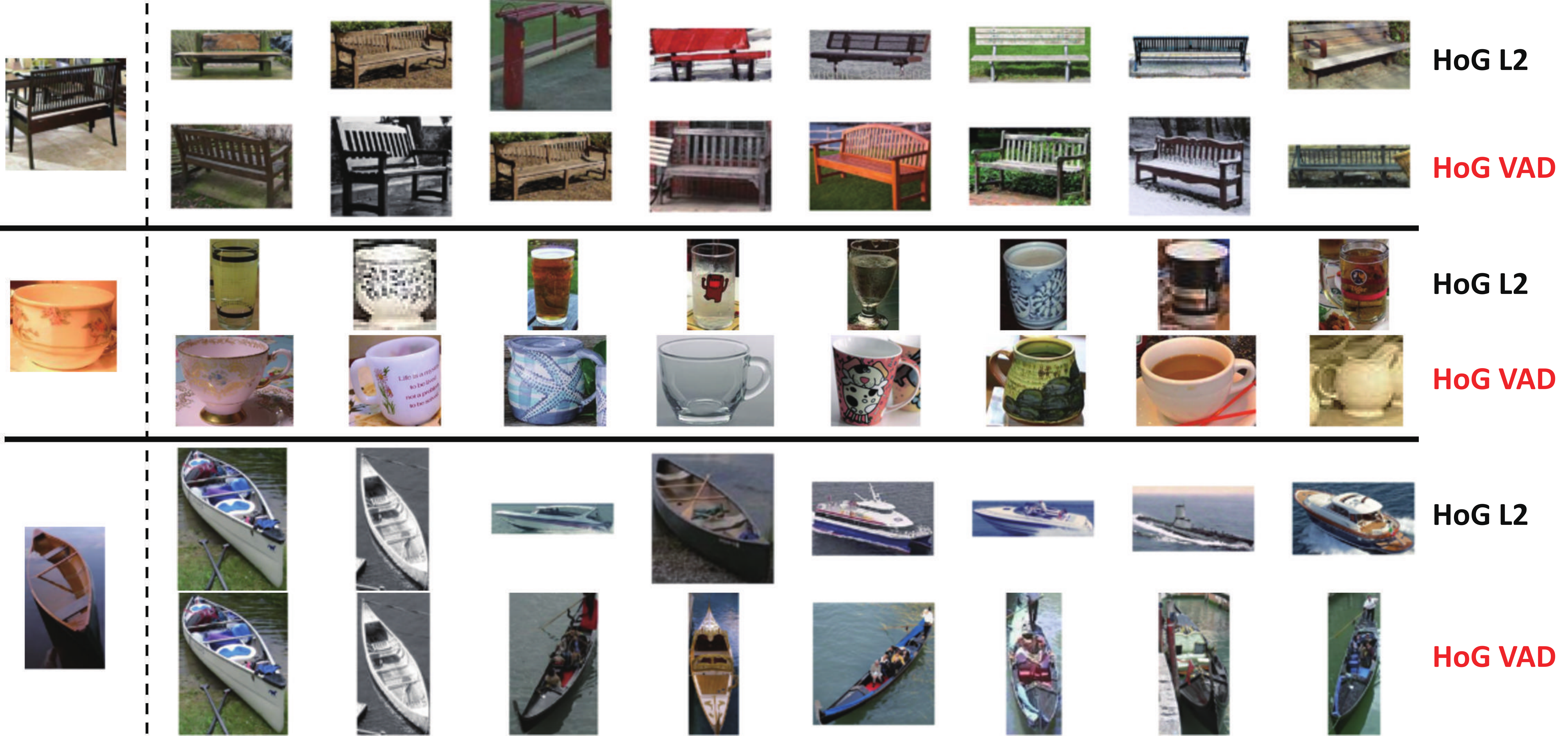}
	\caption{{\bf Fine-grained retrieval results for ``Bench'', ``Cup'' and ``Watercraft''.} \label{fig:retrieval_results}}
\end{figure*}

\subsection{Applications}

\mypara{Fine-grained Image Retrieval on 100 Classes}

We collect images of 100 classes with bounding boxes from ImageNet and verify their fine-grained labels using AMT. We have also preprocessed shape collections of the corresponding classes. Fine-grained image retrieval performance is evaluated as described in \S\ref{sec:analysis}. On average, the baseline L2 distance of HoG descriptor can achieve AUC of 0.635, and our approach can achieve an AUC of 0.694. Fig.~\ref{fig:retrieval_results} shows some examples of retrieval results for comparison. More results can be found in appendix.

\mypara{Part-based Image Retrieval}

Our approach can enable a new application of part-based image retrieval. The user can specify a region on the query image, and our approach can synthesize the features of related patches on novel views. The distance between images will only be evaluated on these patches instead of the whole images. Fig.~\ref{fig:partresults} shows examples of part-based image retrieval. The rectangles on query images are the input specified by users. Although the algorithm can only see the provided patch on the view of query image, it returns images with similar appearance in the corresponding regions from other viewpoints. This part-based search can be useful in product search by image, allowing users to express preferences for product parts.
\begin{table}[b!]
	\centering
	\begin{tabular}{|c|c|c|c|}
		\hline            &  \cite{DBLP:journals/corr/MajiRKBV13} (SPM)  & \cite{DBLP:journals/corr/MajiRKBV13} with b.box & Ours\\
		\hline  Accuracy  &  0.487 & 0.561 & 0.603 \\ \hline
	\end{tabular}
%	\vspace{-1mm}
	\caption{{\bf Accuracy comparison on FGVC-aircraft.}  Note that our results is based on \cite{DBLP:journals/corr/MajiRKBV13} with bounding boxes.}
	\label{tb:classification}
\end{table}

\mypara{Fine-Grained Object Categorization}

Besides image retrieval, we also evaluate our proposed distance on fine-grained object categorization. We use the FGVC-aircraft dataset~\cite{DBLP:journals/corr/MajiRKBV13}, which contains 10,000 images with 100 different aircraft model variants.
We use the non-linear SVM on a $\chi^2$ kernel and replicate the SPM feature setting in \cite{DBLP:journals/corr/MajiRKBV13}, i.e, 600 k-means bag-of-visual words dictionary, multi-scale dense SIFT features, and $1\times 1$, $2\times 2$ spatial pyramid features. Using our approach, we obtain the view-invariant version of SPM feature. We also use bounding boxes predicted by R-CNN~\cite{girshick2014rcnn} and random forest for pose estimation (\S\ref{sec:preprocessing}) on test data. Table~\ref{tb:classification} shows that our method significantly outperforms the baseline. Note that the baseline method in \cite{DBLP:journals/corr/MajiRKBV13} does not use object bounding boxes in testing. To be fair, we also provide the baseline performance with bounding boxes provided.

\begin{figure}[h!]
	\centering
	\includegraphics[width=1\linewidth]{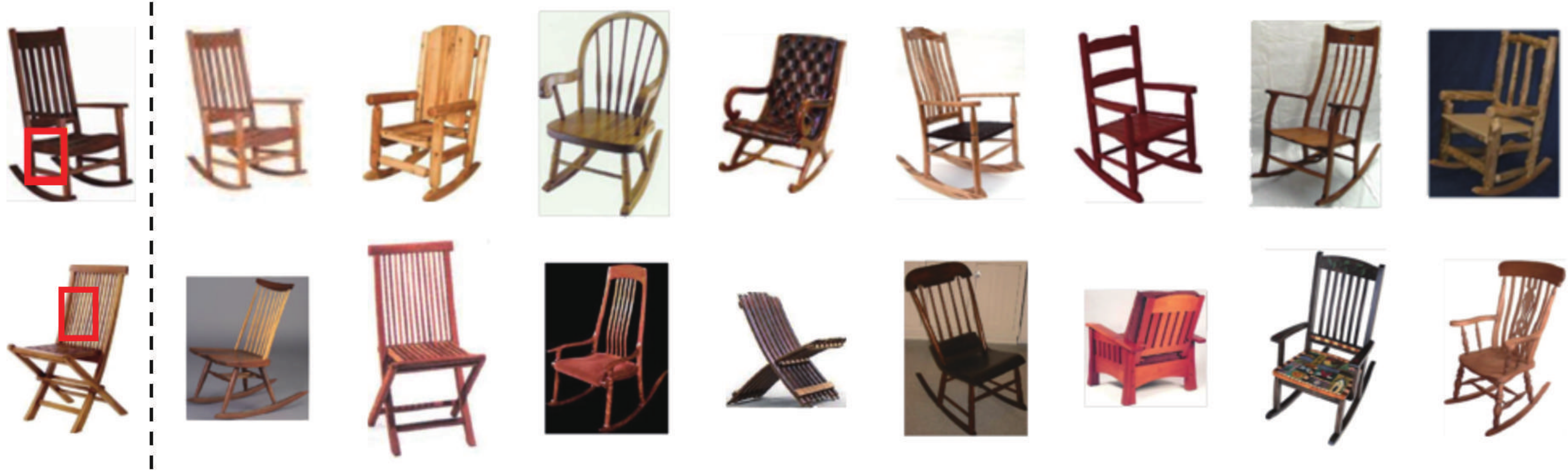}	
	\caption{{\bf Part-based retrieval results} \label{fig:partresults}}
\end{figure}

%% file: conclusion.tex
In this paper, we have proposed a framework for synthesizing features of a new view of an object in an image, given a collection of 3D models from the same object class. By collecting together the features from several canonical views of the object, we arrive at a view-independent model of the object. With this representation, we can achieve view-invariant image comparison, factoring out the influence of viewpoint and only focusing on the intrinsic object properties. The proposed feature synthesis framework has been analyzed theoretically and empirically, and the resulting view-agnostic distance has been evaluated on various computer vision tasks, including fine-grained image retrieval and classification.

%% file: limit.tex
\label{sec:Limit}
\mypara{Future Work}

Our current framework does not take object occlusion and background clutter into consideration.
%Empirical results show that these two factors may severely affect the algorithm performance.
We leave the task as a future work.
In addition, the current surrogate region discovery method is at the categorical level, ignoring details of individual shapes. Geometric properties such as symmetry and part decomposition may help. Then one can use visible parts of an object to predict invisible parts.

%Lastly, our framework transfers information about inter-shape relationship across views. This idea is related with the transfer learning literature~\cite{pan2010survey}. We are interested in further exploring the connection.

%% file: supplementary.tex
%%%%%%%%% BODY TEXT

{\bf Please also refer to the .mp4 video at \url{http://ai.stanford.edu/~haosu/FeatureSynth/video_for_arxiv.mp4} for more visualization and explanation.}

\subsection{Intuition to Weight Transferrability (\S 3.4)}
\begin{hypo}
	If a patch $g_0$ can perfectly surrogate $g_1$, and $\vec x_{v_0, g_0}$ and $\vec x_{v_1, g_1}$ are from some shape $ S$ such that $\vec x_{v_0, g_0}= \mat{S}_{:, v_0, g_0} \vec w_{v_0, g_0}$ and $\vec x_{v_1, g_1} = \mat{S}_{:, v_1, g_1} \vec w_{v_1, g_1}$, then $\vec w_{v_0, g_0} \equiv \vec w_{v_1, g_1}$.
\end{hypo}
	
To provide more intuition, we check the simple cases when $\vec x_{v_0, g_0} = \vec S_{n, v_0, g_0}$ for $1\le n \le N$, i.e., patch $g_0$ at view $v_0$ of shape $\vec S$ looks exactly like the corresponding patch of some $\vec S_n$. $\vec x_{v_0, g_0} = \vec S_{n, v_0, g_0}$ implies that $\vec w_{v_0, g_0} = [0, \cdots, 1, \cdots, 0]^T$, i.e., a zero vector except that the $n$-th element equals  to 1\footnote{More rigorously, the result is true when the null space of $S_{:, v_0, g_0}$ is $\{\vec 0\}$. In practice, the condition is usually satisfied when discriminative features such as HoG and DeepLearning features are used.}. On the other hand, because $g_0$ is a perfect patch surrogate of $g_1$, following the definition, $\vec x_{v_1, g_1} = \vec S_{n, v_1, g_1}$. This implies $\vec w_{v_1, g_1} = [0, \cdots, 1, \cdots, 0]^T$. Therefore, $\vec w_{v_1, g_1} = \vec w_{v_0, g_0}$.

\subsection{Derivation of Patch Surrogate Suitability (\S 3.2.1)}
We first repeat the definition of surrogate suitability:
\begin{align*}
\gamma(g_0; g_1) =& \log P(A_{g_1}^i=A_{g_1}^j|A_{g_0}^i=A_{g_0}^j)\\
=&\log P(A_{g_0}^i=A_{g_0}^j, A_{g_1}^i=A_{g_1}^j)-\log P(A_{g_0}^i=A_{g_0}^j)\\
\end{align*}
Please refer to \S 3.2.1 of main paper for notation definitions.

Next, we introduce the algorithm and sample complexity analysis for estimating $\gamma(g_0; g_1)$.

\begin{lem}
$\gamma(g_0; g_1) = \log \sum_{(A_{g_0}, A_{g_1})\in\mathcal{D}\times\mathcal{D}} P^2(A_{g_0}, A_{g_1})-\log \sum_{A_{g_0}\in\mathcal{D}} P^2(A_{g_0})$
\label{lemma:key}
\end{lem}
\begin{proof}
Computation of the first term:
\begin{align*}
&\log P(A_{g_0}^i=A_{g_0}^j, A_{g_1}^i=A_{g_1}^j) \\
=& \log \sum_{\substack{x\in\mathcal{D} \\y\in\mathcal{D}}} P(A_{g_0}^i=x, A_{g_0}^j=x, A_{g_1}^i=y, A_{g_1}^j=y)\\
=& \log \sum_{(x, y)\in\mathcal{D}\times\mathcal{D}} P((A_{g_0}^i, A_{g_1}^i)=(x, y), (A_{g_0}^j, A_{g_1}^j)=(x, y))\\
=& \log \sum_{(x, y)\in\mathcal{D}\times\mathcal{D}} P((A_{g_0}^i, A_{g_1}^i)=(x, y)) P((A_{g_0}^j, A_{g_1}^j)=(x, y))\\
=& \log \sum_{(x, y)\in\mathcal{D}\times\mathcal{D}} P^2((A_{g_0}, A_{g_1})=(x, y))
\end{align*}
where the last line follows from the independence of Shape $i$ and Shape $j$.

Computation of the second term :
\begin{align*}
\log P(A_{g_0}^i=A_{g_0}^j) &= \log \sum_{x\in\mathcal{D}} P(A_{g_0}^i=x, A_{g_0}^j=x)\\
&= \log \sum_{x\in\mathcal{D}} P(A_{g_0}^i=x)P(A_{g_0}^j=x)\\
&= \log \sum_{x\in\mathcal{D}} P^2(A_{g_0}=x)
\end{align*}
where the last line follows from the independence of Shape $i$ and Shape $j$.
\end{proof}

Note that, $\log \sum_x P^2(x)$ in Lemma~\ref{lemma:key} is a classical quantity in Information Theory, named R\'enyi-entropy. Recent work \cite{DBLP:journals/corr/AcharyaOST14} provides a tight bound for the sample complexity of estimating R\'enyi-entropy. In the following, we restate their results.

Let $f$ be an \emph{estimator} of R\'enyi-entropy $\log \sum_x P^2(X)$ for distributions over support set $\mathcal{X}$, then $f:\mathcal{X}^*\rightarrow\mathbb{R}$ maps a sequence of samples drawn from a distribution $p$ to its entropy. Given independent samples $X^n=X_1, \ldots, X_n$ from $p$, define
$$S^f(k, \delta, \epsilon) \overset{\mbox{def}}=\sup_p\{\min\{n:P(|H(p)-f(X^n)|>\delta)<\epsilon\}\}$$
where $k=|\mathcal{X}|$ is the cardinality of $\mathcal{X}$.

Then, the \emph{sample complexity} of estimating $H(p)$ is defined as
$$S(k, \delta, \epsilon)\overset{\mbox{def}}=\inf_f S^f(k, \delta, \epsilon)$$

\cite{DBLP:journals/corr/AcharyaOST14} shows that $S(k, \delta, \epsilon)=\Theta(\sqrt{k})$. The bound is achievable using the unbiased estimator.

Using the above result, we can prove the following result:
\begin{thm}
	The optimal sample complexity to estimate $\gamma(g_0; g_1)$ is $\Theta(D)$, where $D$ is the cardinality of symbol set $\mathcal{D}$ as defined in \S 3.2.1, i.e., the number of visual words in dictionary.
\end{thm}
\begin{proof}
It is easy to see that the sample complexity to estimate the first term of $\gamma$ in Lemma~\ref{lemma:key} is $\Theta(D)$ and the second term is $\Theta(\sqrt{D})$. Therefore, the overall sample complexity for estimating $\gamma$ is $\Theta(D)$.
\end{proof}

\subsection{More Results on Image Retrieval}
Next two pages show groups of representative image retrieval results using our view-agnostic distance and the baseline L2 HoG feature distance. For each group, the first image is the query image.
\begin{figure}
	\centering
	\includegraphics[width=1\linewidth]{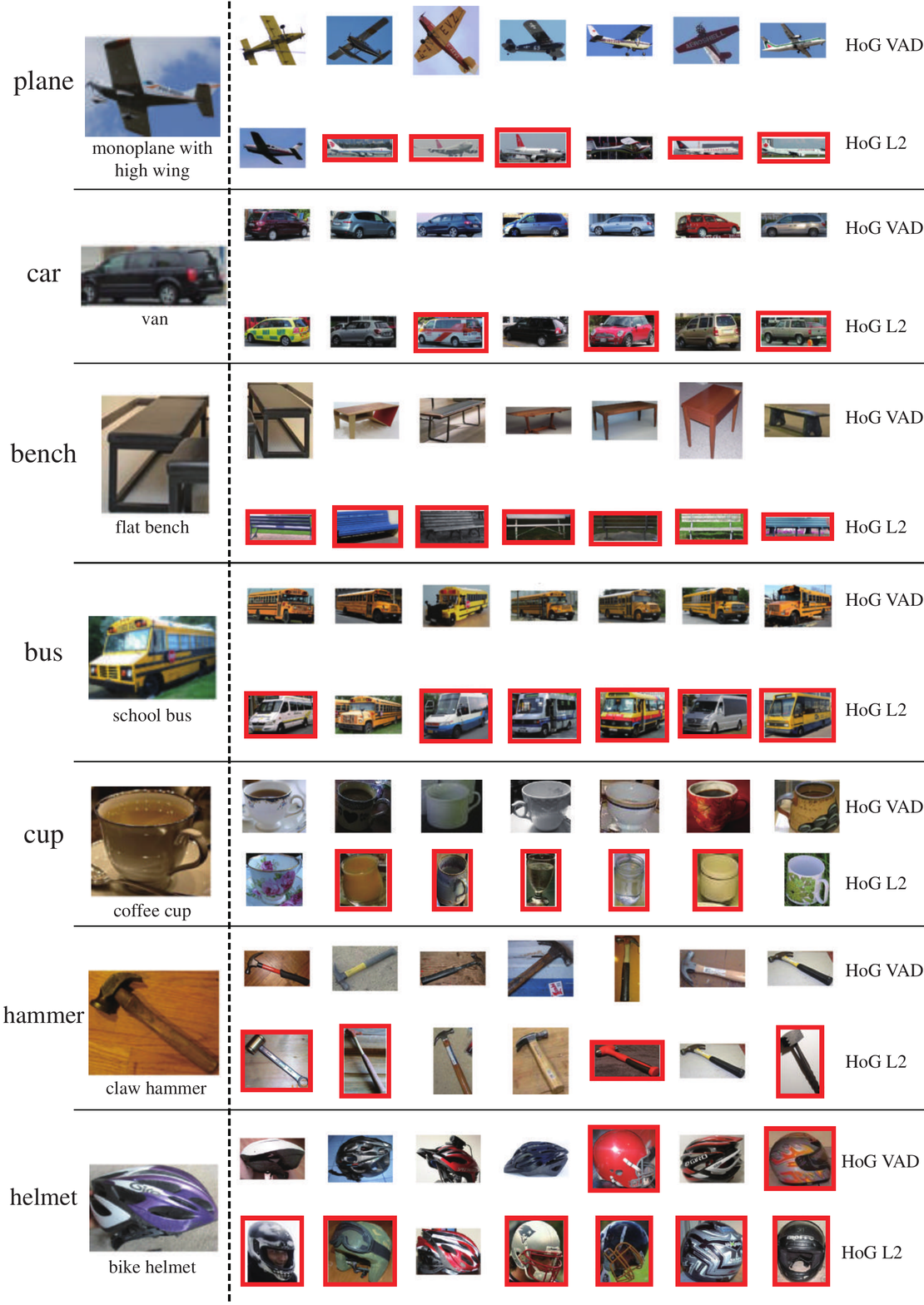}
	\caption{Image retrieval results on 100 classes (to be continued...)}
\end{figure}
\begin{figure}
	\centering	
	\includegraphics[width=1\linewidth]{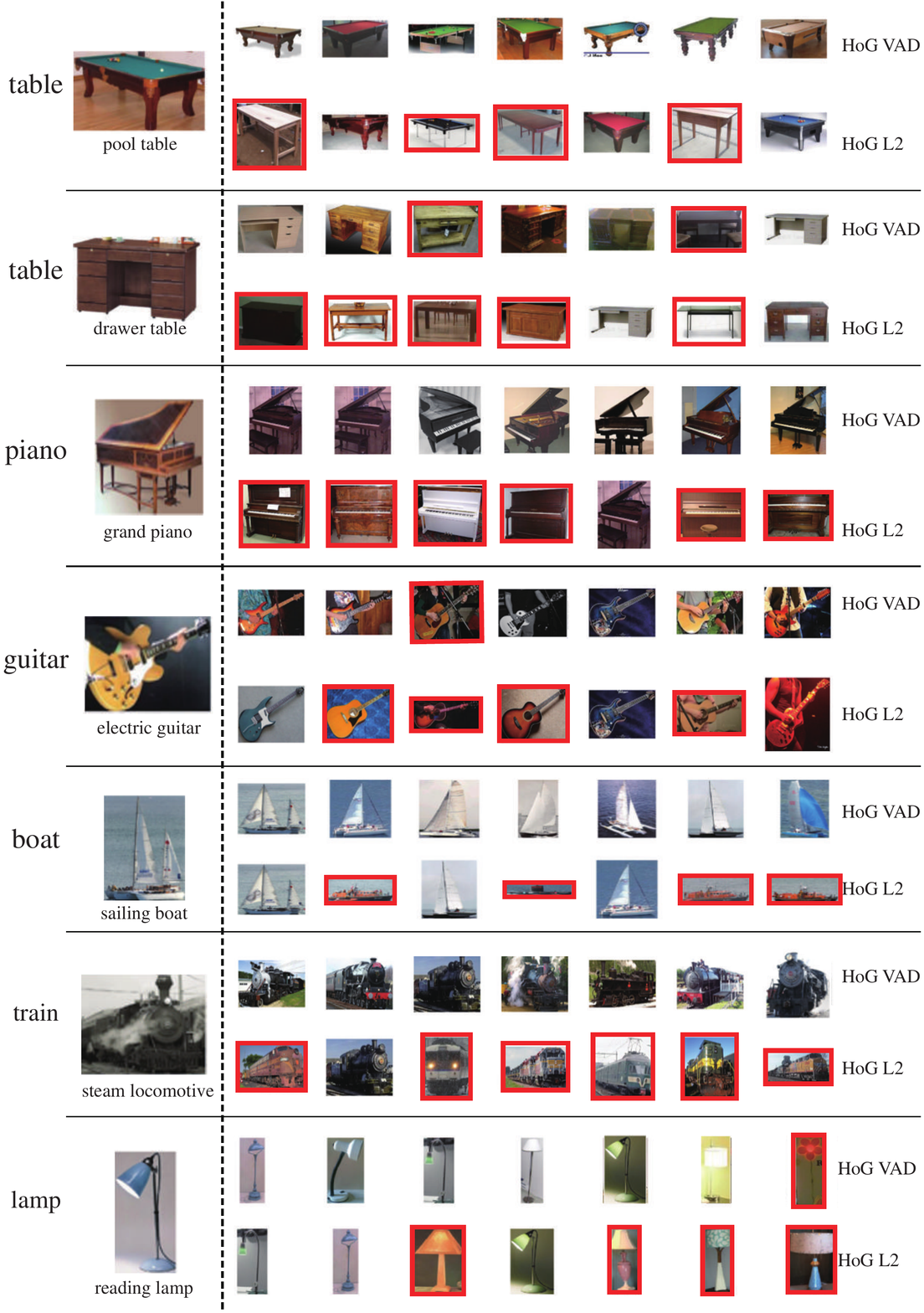}
	\caption{Image retrieval results on 100 classes}
\end{figure}